\documentclass[runningheads]{llncs}
\usepackage[T1]{fontenc}
\usepackage{graphicx}

\usepackage{amsxtra, amsfonts, amssymb, amstext, amsmath, mathtools}
\usepackage{dsfont}
\usepackage{booktabs}
\usepackage{nicefrac}
\usepackage{xspace}
\usepackage{graphicx, color}
\usepackage[algo2e,ruled,vlined,linesnumbered]{algorithm2e}\SetArgSty{upshape}
\usepackage[square,numbers,sort&compress]{natbib}
\usepackage{hyperref}

\usepackage[noabbrev,nameinlink]{cleveref}

\newcommand{\NSGA}{\mbox{NSGA}\nobreakdash-II\xspace}
\newcommand{\moead}{MO\-EA/D\xspace}

\newcommand{\oneminmax}{\textsc{OneMinMax}\xspace}
\newcommand{\omm}{\mathrm{OMM}}

\newcommand{\lotz}{\mathrm{LOTZ}}

\newcommand{\R}{\ensuremath{\mathbb{R}}}

\newcommand{\N}{\ensuremath{\mathbb{N}}} % ohne Null!!!

\DeclareMathOperator{\E}{E} %use with [...]
\DeclareMathOperator{\Var}{Var}
\newcommand{\eulerE}{\mathrm{e}}

\let\originalleft\left
\let\originalright\right
\renewcommand{\left}{\mathopen{}\mathclose\bgroup\originalleft}
\renewcommand{\right}{\aftergroup\egroup\originalright}

\begin{document}

\title{Proven Runtime Guarantees for How the \moead Computes the Pareto Front From the Subproblem Solutions}
\titlerunning{\moead: Computing the Pareto Front from the Subproblem Solutions}
%\title{Exponential Speed-Up of the \moead on OneMinMax When Using Power-Law Mutation}
%
%\titlerunning{Abbreviated paper title}
% If the paper title is too long for the running head, you can set
% an abbreviated paper title here
%
\author{Benjamin Doerr\inst{1}\orcidID{0000-0002-9786-220X} \and
Martin~S. Krejca\inst{1}\orcidID{0000-0002-1765-1219} \and
Noé Weeks\inst{2} %\orcidID{2222-3333-4444-5555}
}
%
%\authorrunning{Benjamin.~Doerr et al.}
% First names are abbreviated in the running head.
% If there are more than two authors, 'et al.' is used.
%
\institute{Laboratoire d'Informatique (LIX), CNRS, \'Ecole Polytechnique, Institut Polytechnique de Paris, Palaiseau, France\and
École Normale Supérieure}

\maketitle              % typeset the header of the contribution
\begin{abstract}
% The abstract should briefly summarize the contents of the paper in
% 150--250 words.

The decomposition-based multi-objective evolutionary algorithm (\moead) does not directly optimize a given multi-objective function~$f$, but instead optimizes $N + 1$ single-objective subproblems of~$f$ in a co-evolutionary manner. It maintains an archive of all non-dominated solutions found and outputs it as approximation to the Pareto front.
Once the \moead found all optima of the subproblems (the $g$-optima), it may still miss Pareto optima of~$f$.
The algorithm is then tasked to find the remaining Pareto optima directly by mutating the $g$-optima.

In this work, we analyze for the first time how the \moead with only standard mutation operators computes the whole Pareto front of the \oneminmax benchmark when the $g$-optima are a strict subset of the Pareto front. For standard bit mutation, we prove an expected runtime of $O(n N \log n + n^{n/(2N)} N \log n)$ function evaluations.
Especially for the second, more interesting phase when the algorithm start with all $g$-optima, we prove an $\Omega(n^{(1/2)(n/N + 1)} \sqrt{N} 2^{-n/N})$ expected runtime.
This runtime is super-polynomial if $N = o(n)$, since this leaves large gaps between the $g$-optima, which require costly mutations to cover.

For power-law mutation with exponent $\beta \in (1, 2)$, we prove an expected runtime of $O\left(n N \log n + n^{\beta} \log n\right)$ function evaluations.
The $O\left(n^{\beta} \log n\right)$ term stems from the second phase of starting with all $g$-optima, and it is independent of the number of subproblems~$N$.
This leads to a huge speedup compared to the lower bound for standard bit mutation.
In general, our overall bound for power-law suggests that the \moead performs best for $N = O(n^{\beta - 1})$, resulting in an $O(n^\beta \log n)$ bound.
In contrast to standard bit mutation, smaller values of~$N$ are better for power-law mutation, as it is capable of easily creating missing solutions.

\end{abstract}

% We studied the decomposition-based evolutionary algorithm for multi-objective optimization (\moead)~\cite{ZhangL07} on the classic bi-objective \oneminmax problem~\cite{GielL10} of problem size~$n$.
% Instead of directly optimizing an objective function~$f$, the \moead optimizes $N + 1$ single-objective problems, each based on~$f$.
% We investigated the impact of choosing~$N$ smaller than the size of the Pareto front of~$\omm$, which is $n + 1$.
% We considered an even spread of the subproblems of the \moead, and we analyzed the algorithm's runtime both theoretically and empirically.

\section{Introduction}

Many real-world problems require the simultaneous optimization of different objectives.
In this setting, known as \emph{multi-objective optimization}, different solutions may not be comparable based on their objective values, as one solution can win over another solution in one objective, but lose in another objective.
This results in a set of incomparable optimal objective values, commonly referred to as \emph{Pareto front}.
The aim in multi-objective optimization is to find the Pareto front of a problem, or a good approximation thereof.

Due to their population-based and heuristic nature, evolutionary algorithms lend themselves very well to multi-objective optimization, and they have been successfully applied for decades to a plethora of hard multi-objective optimization problems~\cite{ZhouQLZSZ11}.
This strong interest has led to a variety of algorithms~\cite{ZhangX17MOSurvey,ZhouQLZSZ11}, following different paradigms.

From the early days of theoretical analyses of evolutionary algorithms on, multi-objective evolutionary algorithms have been analyzed also via theoretical means~\cite{Rudolph98ep,LaumannsTZ04,Giel03,Thierens03}. This area saw a boost of activity in the last two years, when the first mathematical runtime analysis of the \NSGA~\cite{ZhengLD22} inspired many deep analyses of this important algorithm and variants such as the NSGA-III or SMS-EMOA \cite{ZhengD22gecco,BianQ22,DoerrQ23tec,DangOSS23aaai,DoerrQ23LB,DoerrQ23crossover,DangOSS23gecco,CerfDHKW23,WiethegerD23,ZhengD23aij,ZhengD23many,ZhengD24,ZhengLDD24,OprisDNS24}. %\todo{Probably overkill? B: Viel hilft viel ;-)}
%Most of these works study the non-dominated sorting genetic algorithm II (NSGA-II)~\cite{DebJ14}, which is the most popular multi-objective evolutionary algorithm. It follows the paradigm of solving a problem directly by applying operators specialized for the multi-objective domain, such as selecting favorable solutions among sets that contain potentially incomparable solutions.

A substantially different, yet also very important algorithm is the \emph{decompo\-sition-based multi-objective evolutionary algorithm} (\moead)~\cite{ZhangL07}. It decomposes a multi-objective optimization problem~$f$ into various single-objective subproblems of~$f$.
These subproblems are optimized in parallel.
While doing so, the non-dominated solutions for~$f$ are maintained in an archive.

Despite its popularity in empirical research and good performance in real-world problems~\cite{ZhangX17MOSurvey,ZhouQLZSZ11}, the \moead has not been extensively studied theoretically~\cite{LiZZZ16,HuangZCHLX21,HuangZ20,DoNNS23MOEADforCombinatorics}. In particular, it is currently not known how the basic \moead using only standard mutation operators as variation operators finds Pareto optima that are not already an optimum of one of the subproblems (we refer to \Cref{sec:related-work} for a detailed discussion of the previous works). We recall that the \moead solves a number of single-objective subproblems essentially via single-objective approaches. Hence, if all Pareto optima of the original problem appear as optima of subproblems (we call these \emph{$g$-optima} in the remainder), this is the setting regarded, e.g., in~\cite{LiZZZ16}, then the multi-objective aspect of the problem vanishes and the only question is how efficiently the single-objective approaches solve the subproblems.

% Especially, it is currently not well understood how strongly the algorithm's runtime for standard mutation operators is affected by the number of subproblems to consider.
% More specifically, if the \moead optimizes only a few subproblems and found their solutions (which we call \emph{$g$-optima}), it is unlikely that the $g$-optima cover the Pareto front of the multi-objective problem sufficiently well.
% The algorithm then requires to find the remaining Pareto-optimal solutions by immediately turning $g$-optima into other Pareto-optimal solutions, as the $g$-optima will not be replaced, due to their optimality.
% In this work, we study mathematically how the performance of the \moead is affected by the choice of the number~$N$ of subproblems to consider.

% We review the limited related work in more detail in .

\subsubsection*{Our Contribution.}
Naturally, in a typical application of the \moead, one cannot assume that the subproblems are sufficiently numerous and evenly distributed so that each Pareto optimum appears as $g$-optimum. To better understand how the \moead copes with such situations, we study in this work mathematically how the \moead computes the full Pareto front when started with a population consisting of all $g$-optima. In this first work on this topic, as often in the mathematical runtime analysis, we consider the basic \oneminmax ($\omm$) benchmark~\cite{GielL10}. As in most previous works, we assume that the $N+1$ subproblems are constructed in a manner that the corresponding $g$-optima are spread equidistantly, with respect to the Hamming distance, across the Pareto front. We regard the basic \moead using standard bit mutation as only variation operator.
Interestingly, this is the first theoretical analysis of the \moead in this setting (apart from the case that the $g$-optimal cover the Pareto front~\cite{LiZZZ16}). Hence our results (\Cref{thm:second-phase-standard-bit}), together with standard estimates on this time to compute the $g$-optima (\Cref{thm:first-phase}), also give the first estimates on the full runtime of the \moead in this setting (\Cref{thm:overall-run-time-standard-bit}).

Since our results show that the main bottleneck for computing points on the Pareto front that are far from all $g$-optima is that standard bit mutation rarely flips many bits (due to the concentration behavior of the binomial distribution), we also resort to the heavy-tailed mutation operator proposed by Doerr et~al.~\cite{DoerrLMN17}. Interestingly, this allows for drastic speed-ups when considering the phase of the optimization that starts with all $g$-optima (\Cref{thm:phase-two}).

In detail, our various results provide us with expected-runtime estimates for the \moead with either mutation operator to optimize~$\omm$ (\Cref{thm:overall-run-time-standard-bit,thm:overall-run-time-power-law}).
These results prove, respectively, an expected number of $O(n N \log n + n^{n/(2N)} N \log n)$ function evaluations for the \moead with standard bit mutation, and $O\left(n N \log n + n^{\beta} \log n\right)$ expected function evaluations for power-law mutation with power-law exponent $\beta \in (1, 2)$.
In both results, the second term refers to the interesting phase where the algorithm is initialized with all $g$-optima and is tasked to find the remaining Pareto optima of~$\omm$.

Our overall bound for standard bit mutation yields $O(n^2 \log n)$ in the best case of $N = n$, matching the result by  Li et~al.~\cite[Proposition~$4$]{LiZZZ16}.
For general~$N$, the second term in our bound suggests that the \moead performs best if $N \in [\frac{n}{2} .. n]$ and that the runtime is super-polynomial once $N = o(n)$.
Moreover, we prove a lower bound of $\Omega(n^{(1/2)(n/N + 1)} \sqrt{N} 2^{-n/N})$ for this second term (\Cref{thm:second-phase-standard-bit}), which supports this runtime characterization.
However, this lower bound is not necessarily applicable to the \emph{entire} optimization of the \moead on~$\omm$, as we only prove it for a certain phase of the optimization process.
Nonetheless, we believe that it is true for sufficiently large~$n$.
We go more into detail about this behavior in \Cref{sec:theory} and especially in \Cref{sec:theory:first-phase}.
Overall, our bounds suggest that standard bit mutation performs better when~$N$ is large.

Our upper bound for power-law mutation is best once $N = O(n^{\beta - 1})$, resulting in a runtime bound of $O(n^\beta \log n)$.
The bound $O(n^\beta \log n)$ stems from the second phase of the optimization, where the algorithm is initialized with all $g$-optima (\Cref{thm:phase-two}).
It is noteworthy that this bound does not depend on~$N$, as, roughly, the time to fill in the gaps between $g$-optima is inversely proportional to the cost of performing $N + 1$ function evaluations each iterations.
Hence, the parameter~$N$ only influences our bound for the first phase of finding all $g$-optima.
Overall, this suggests that the \moead performs better when~$N$ is small, which is opposite of the behavior with standard bit mutation.
Moreover, the bound for power-law mutation is $O(n^2 \log n)$ in the worst case, matching the best case of our bound for standard bit mutation.
This suggests that power-law mutation is more preferable than standard bit mutation in our setting, as power-law mutation exhibits a far more robust runtime behavior.

Last, for each of the two phases we consider, we get independent results that hold for a larger class of algorithms (\Cref{lem:reference-point-optimization,lem:finding-all-g-optima}) or functions (\Cref{thm:second-phase-standard-bit,thm:phase-two}).
Moreover, we prove an anti-concentration bound for the hypergeometric distribution close around its expected value (\Cref{lem:hypergeo}).
Due to the more general setting, we believe that these results are of independent interest.

\section{Related Work}
\label{sec:related-work}
Theoretical analyses of the \moead so far are scarce.
Most results do not consider the \moead with only standard mutation operators.
And those that do make simplifying assumptions about the decomposition, using problem-specific knowledge.
We also note that we could not find any proofs for how the \moead finds its reference point (one of its parameters), which is important in the general scenarios.
If the reference point is mentioned, it is immediately assumed to be best-possible.

Li et~al.~\cite{LiZZZ16} conducted the first mathematical analysis of the \moead.
They study the runtime of the algorithm on the classic bi-objective \oneminmax ($\omm$)~\cite{GielL10} and \textsc{LeadingOnesTrailingZeros} ($\lotz$)\footnote{We note that the authors call~$\omm$ COCZ, and that they consider a version of~$\lotz$, called \textsc{LPTNO}, that considers strings over~$\{-1, 1\}$ instead of binary strings, effectively replacing~$0$s with~$-1$s. This changes some values for certain parameters but effectively results in the same bounds, to the best of our knowledge. Hence, we use the names~$\omm$ and~$\lotz$ throughout this section.} benchmarks of size~$n$, that is, the number of objective-function evaluations until the \moead finds all Pareto optima of the problem.
Both benchmarks have a Pareto front of size $n + 1$.
The authors consider a decomposition of each problem into $n + 1$ subproblems---matching the size of the Pareto front---, and they assume that the \moead uses standard bit mutation.
The authors prove that if the subproblems are chosen such that the $n + 1$ $g$-optima correspond one-to-one to the Pareto front of the problem, then the \moead optimizes~$\omm$ in $O(n^2 \log n)$ expected function evaluations, and~$\lotz$ in $O(n^3)$.
We note that this requires a different subproblem structure for~$\omm$ and~$\lotz$.
Moreover, the authors show that if the \moead uses the commonly used subproblem structure of~$\omm$ for~$\lotz$, then the $g$-optima do not cover the entire Pareto front of~$\lotz$, for sufficiently large~$n$.
In this case, the authors argue that optimization takes a long time, as the missing solutions need to be found, but these arguments are not made formal.

Huang and Zhou~\cite{HuangZ20} analyze the \moead when using contiguous hypermutations, a non-standard mutation operator stemming from artificial immune systems.
The authors study two versions of contiguous hypermutation, and they consider the~$\omm$ and the~$\lotz$ benchmarks as well as a deceptive bi-objective problem and one containing a plateau.
Moreover, the authors consider a general decomposition of each problem into $N + 1$ subproblems.
This decomposition is always the same and corresponds to the commonly used one, which results in evenly spread $g$-optima on the Pareto front of~$\omm$, also analyzed by Li et~al.~\cite{LiZZZ16} above.
Huang and Zhou~\cite{HuangZ20} prove that both \moead variants optimize each of the four benchmarks in $O(N n^2 \log n)$ expected function evaluations.
Moreover, they prove the same runtime bounds for the $4$-objective versions of~$\omm$ and~$\lotz$.
Overall, their results suggest that a choice of $N = O(1)$ is more beneficial, as the hypermutation operators are capable to cover larger distances than standard bit mutation and can thus find Pareto optima efficiently that do not correspond to $g$-optima.

Huang et~al.~\cite{HuangZCHLX21} analyze a variant of the \moead that employs standard bit mutation as well as one-point crossover at a rate of~$0.5$.
When performing crossover, one of the parents is the best solution of the current subproblem~$g$ and the other one is chosen uniformly at random among the non-dominated solutions of the subproblems closest to~$g$ with respect to the Euclidean distance in one of their parameters.
The problem decomposition is always the commonly used one with equidistant $g$-optima, as in the two papers discussed above.
The authors consider essentially the same four\footnote{The authors actually treat~$\lotz$ and \textsc{LPTNO} as two functions, but the results are the same, which is why we count it as one problem.} problems as Huang and Zhou~\cite{HuangZ20} above, and they also consider a general number of $N + 1$ subproblems.
For~$\lotz$ and the plateau function, the authors proved an expected runtime of $O(N n^2)$ function evaluations, and an $O(N n \log n)$ bound for~$\omm$ and the deceptive function.

Very recently, Do et~al.~\cite{DoNNS23MOEADforCombinatorics} analyzed the \moead on the multi-objective minimum-weight-base problem, which is an abstraction of classical NP-hard combinatorial problems.
Their \moead variant uses a decomposition based on weight scalarization, different from the previous works above.
The authors then prove that this variant finds an approximation of the Pareto front of the problem within expected fixed-parameter polynomial time.

\section{Preliminaries}
We denote the natural numbers by~$\N$, including~$0$, and the real numbers by~$\R$.
For $a, b\in \R$, let $[a..b] = [a, b]\cap \N$ and $[a] = [1..a]$.

Let $n \in \N_{\geq 1}$.
We consider \emph{bi-objective} optimization problems, that is, functions $f\colon \{0, 1\}^n \to \R^2$.
We always assume that the dimension $n \in \N_{\geq 1}$ is given implicitly.
When using big-O notation, it refers to asymptotics in this~$n$.
In this sense, an event occurs \emph{with high probability} if its probability is at least $1 - o(1)$.

We call a point $x\in\{0, 1\}^n$ an \emph{individual} and $f(x)$ the \emph{objective value} of $x$.
For all $i \in [n]$ and $j \in [2]$, we let $x_i$ denote the $i$-th component of $x$ and $f_j(x)$ the $j$-th component of $f(x)$.
Moreover, let $|x|_0$ denote the number of~$0$s of~$x$, and let $|x|_1$ denote its number of~$1$s.

For all $u, v \in \R^2$, we say that~$v$ \emph{weakly dominates} $u$ (written $v \succeq u$) if and only if for all $i \in [2]$ holds that $f_i(v) \ge f_i(u)$.
We say that~$v$ \emph{strictly dominates} $u$ if and only if one of these inequalities is strict.
We extend this notation to individuals, where a dominance holds if and only if it holds for their respective objective values.

We consider the maximization of bi-objective functions~$f$, that is, we are interested in $\succeq$-maximal elements, called \emph{Pareto-optimal} individuals.
The set of all objective values that are not strictly dominated, that is, the set $F^* \coloneqq \{v \in \R^2 \mid \not\exists x\in \{0, 1\}^n, f(x) \succ v\}$, is called the \emph{Pareto front} of~$f$.

\subsubsection*{\oneminmax}
We analyze the \oneminmax ($\omm$) benchmark~\cite{GielL10} problem, which returns for each individual the number of~$0$s as the first objective, and the number of~$1$s as the second objective.
Formally, $\omm\colon x \mapsto (|x|_0, |x|_1)$.

Note that each individual is Pareto optimal. The Pareto front of \oneminmax is $\{(i, n-i), i \in [0..n]\}$.

\subsubsection*{Mathematical Tools}

We use the well-known multiplicative drift theorem~\cite{DoerrJW12algo} with tail bounds~\cite{DoerrG13algo}.
We present the theorem in a fashion that is sufficient for our purposes.
Throughout this article, if we write for a stopping time~$T$ and two unary formulas~$P$ and~$Q$ that for all $t \in \N$ with $t < T$ holds that $P(t) \geq Q(t)$, then we mean that for all $t \in \N$ holds that $P(t) \cdot \mathds{1}\{t < T\} \geq Q(t) \cdot \mathds{1}\{t < T\}$, where~$\mathds{1}$ denotes the indicator function.

\begin{theorem}[Multiplicative drift~{\cite{DoerrJW12algo}}, upper tail bound~{\cite{DoerrG13algo}\cite[Theorem~$2.4.5$]{Lengler17}}]
    \label{thm:multiplicative-drift}
    Let $n \in \N$, let $(X_t)_{t \in \N}$ be a random process over $[0 .. n]$, and let $T = \inf \{t \in \N \mid X_t = 0\}$.
    Moreover, assume that there is a $\delta \in \R_{> 0}$ such that for all $t \in \N$ with $t < T$ holds that $\E[X_t - X_{t + 1} \mid X_t] \geq \delta X_t$.
    Then, $\E[T] \leq \frac{1}{\delta} \ln n$, and for all $r \in \R_{\geq 0}$ holds that $\Pr[T > \frac{1}{\delta}(r + \ln n)] \leq \eulerE^{-r}$.
\end{theorem}

\section{The \moead}
\label{sec:moead}
We analyze the \emph{decomposition-based multi-objective evolutionary algorithm} (\moead; \Cref{alg:moead})~\cite{ZhangL07} for multi-objective optimization.
The \moead decomposes its objective function into a pre-specified number of single-objective subproblems.
These subproblems are optimized in parallel. %, and solutions among different subproblems are exchanged with respect to a neighborhood structure specified as input.
The \moead maintains an archive of the $\succeq$-best solutions found so far, allowing it to find Pareto-optimal solutions for the original problem while only explicitly optimizing the subproblems.

More formally, given, besides others, a bi-objective optimization problem~$f$ as well as a \emph{decomposition number} $N \in [n]$ and a \emph{weight vector} $w \in [0, 1]^{N + 1}$, the \moead optimizes~$f$ by decomposing~$f$ into $N + 1$ single-objective subproblems $\{g_i\colon \{0, 1\}^n \times \R^2 \to \R\}_{i \in [0 .. N]}$, weighted by~$w$.
Each of these subproblems is subject to \emph{minimization} (of an error).
The \moead maintains a population of $N + 1$ individuals $(x_i)_{i \in [0 .. N]} \in (\{0, 1\}^n)^{N + 1}$ as well as a \emph{reference point} $z^* \in \R^2$ such that for each $i \in [0 .. N]$, individual~$x_i$ is the currently best solution found for subproblem~$i$ with respect to~$z^*$.
Ideally, the reference point is a point such that for all $j \in [2]$, value~$z^*_j$ is optimal for objective~$f_j$.
To this end, the \moead updates~$z^*$ whenever it optimizes a subproblem.
Moreover, the algorithm maintains a set $P \subseteq \R^2$ (the \emph{Pareto front}) of non-dominated objective values.

We consider subproblems that measure the maximum distance to the reference point, known as Chebyshev approach.
That is, for all $i \in [0 .. N]$, $x \in \{0, 1\}^n$, and $z^* \in \R^2$, it holds that
\begin{align}
    \label{eq:moead-subproblem}
    g_i(x, z^*) = \max \bigl(w_i \cdot |z^*_1 - f_1(x)|, (1 - w_i) \cdot |z^*_2 - f_2(x)|\bigr) .
\end{align}

When minimizing subproblem $i \in [0 .. N]$, the \moead picks~$x_i$ as parent and mutates it according to a given mutation operator.
Afterward, it compares the offspring of the mutation to~$x_i$ and selects the better one.\footnote{We note that the general \moead allows to specify neighborhoods among the subproblems, which exchange non-dominated solutions among each other. In this article, we focus on no exchange.} % all individuals in the population given by a neighborhood set $\Gamma_i \subseteq [0 .. N]$ and selects the better one each time.

We define the \emph{runtime} of the \moead on~$f$ as the number of function evaluations of~$f$ until the Pareto front~$P$ of the algorithm is equal to the Pareto front~$F^*$ of~$f$ for the first time.

In this article, we consider the \moead with different mutation operators. %and with different neighborhood sets.

\begin{algorithm2e}[t]%
    \caption{\label{alg:moead}
        The \moead~\cite{ZhangL07} maximizing a bi-objective problem $f\colon \{0, 1\}^n \to \R^2$.
        See also \Cref{sec:moead}.
    }
    \KwIn{A decomposition number $N \in \N_{\geq 1}$, a weight vector $w \in [0, 1]^{N + 1}$, subproblems $\{g_i\}_{i \in [0 .. N]}$, %neighborhood sets $\{\Gamma_i\}_{i \in [0 .. N]}$,
        a mutation operator $\mathrm{mut}\colon \{0, 1\}^n \to \{0, 1\}^n$, and a termination criterion.
    }
% \emph{Input}: The power-law exponent $\beta>1$ to use, $f:\{0,1\}^n\to \R^d$ the function to minimize, the number of subproblems $K=N+1$ considered as well as the objective functions $g_i$ for $i\in[0..N]$\;
    \emph{Initialization}: for each $i \in [0 .. N]$, choose~$x_i$ uniformly at random from $\{0, 1\}^n$; set $z^*_1 = \max_{i \in [0 .. N]} f_1(x_i)$, $z^*_2 = \max_{i \in [0 .. N]} f_2(x_i)$, and iteratively add $f(x_i)$ to~$P$ if there is no $j \in [0 .. i - 1]$ such that~$x_i$ is weakly dominated by~$x_j$\;

    \While{stopping criterion is not met}{
        \For{each subproblem $i \in [0..N]$}{
            \emph{Mutation}: $y \gets \mathrm{mut}(x_i)$\;
            \emph{Update~$z^*$}: set $z^*_1 \gets \max(z^*_1, f_1(y))$, $z^*_2 \gets \max(z^*_2, f_2(y))$\;
            % \emph{Update neighbors}: for each $j \in \Gamma_i$, if $g_j(y, z^*) \leq g_j(x_j, z^*)$, then $x_j \gets y$\;
            \emph{Update~$x_i$}: if $g_i(y, z^*) \leq g_i(x_i, z^*)$, then $x_i \gets y$\;
            \emph{Update $P$}: remove all elements weakly dominated by~$f(y)$ from~$P$ and add~$f(y)$ to~$P$ if it is not weakly dominated by an element of $P$\;
        }
    }
\end{algorithm2e}

% \subsection{The Weight Vector}
% Also note that when the weight vectors are defined by $w_i = (1-i/N, i/N)$, then $g_i$ has as optima individuals with $i$ 0s, this will be used extensively.

% \subsubsection*{Neighborhood Sets}
% We consider the \emph{single neighborhood}, where for each $i \in [0 .. N]$ holds that $\Gamma_i = \{i\}$, as well as the \emph{complete neighborhood}, where for each $i \in [0 .. N]$ holds that $\Gamma_i = [0 .. N]$.

\subsubsection*{Mutation Operators}
We consider both \emph{standard bit mutation} as well as \emph{power-law mutation}~\cite{DoerrLMN17}.
Let $x \in \{0 , 1\}^n$ be the input (the \emph{parent}) of the mutation.
Both operators create a new individual (the \emph{offspring}) by first copying~$x$ and then adjusting its bit values.
Standard bit mutation flips, for all $i \in [n]$, bit~$x_i$ independently with probability~$1/n$.

Power-law mutation requires a parameter $\beta \in (1, 2)$ (the \emph{power-law exponent}) as input and utilizes the power-law distribution~$\mathrm{Pow}(\beta, n)$ over $[n]$, defined as follows.
Let $C_\beta = \sum_{i \in [n]} i^{-\beta}$ as well as $X \sim \mathrm{Pow}(\beta, n)$. For all $i \in [n]$, it holds that $\Pr[X = i] = i^{-\beta} / C_\beta$.
The power-law mutation first draws $X \sim \mathrm{Pow}(\beta, n)$ (the \emph{mutation strength}) and then flips an $X$-cardinality subset of positions in~$x$ chosen uniformly at random.
% The last step of this operator corresponds to the \emph{hypergeometric distribution}, which we discuss in more detail in the following.

The following lemma bounds the probability to mutate an individual with at most~$\frac{n}{4}$ $0$s into one with at most~$\frac{n}{2}$, showing that the probability is proportional to the distance in the number of~$0$s.
Its proof makes use of the anti-concentration of the \emph{hypergeometric distribution} (\Cref{lem:hypergeo}) around its expectation, which we discuss after \Cref{lem:trans_prob}.
We note that, due to space restrictions, the proofs of these results are found in the appendix.

% We first prove that we can mutate an individual with $u<n/2$ 0-bits to an individual with $u < v \le n/2$ 0-bits with reasonable probability compared to $v-u$. This is heavily based on \Cref{lem:hypergeo}.

\begin{lemma}\label{lem:trans_prob}
    Let $x, y \in \{0, 1\}^n$ with $u \coloneqq |x|_0 \in [0 .. \frac{n}{4}]$ and $v \coloneqq |y|_0 \in [u + 1 .. \frac{n}{2} - 1]$ with $u < v$.
    Moreover, let $\beta \in (1, 2)$, and let $\mathrm{mut}_\beta$ denote the power-law mutation with power-law exponent~$\beta$.
    Then $\Pr[\mathrm{mut}_\beta(x) = y] = \Omega\bigl((v-u)^{-\beta}\bigr)$.
    % If $0 < u < v < n/2$ with $u\le n/4$, then it is possible to mutate an individual with $u$ 0-bits to an individual with $v$ 0-bits using a power law with probability at least
    % $$ \Theta\left(\left(v-u\right)^{-\beta}\right) $$
\end{lemma}

\paragraph{Hypergeometric distribution.}
The hypergeometric distribution ($\mathrm{Hyp}$) takes three parameters, namely, $n \in \N$, $k \in [0 .. n]$, and $r \in [0 .. n]$, and it has a support of $[\max(0, r + k - n) .. \min(k, r)]$.
A random variable $X \sim \mathrm{Hyp}(n, k, r)$ describes the number of good balls drawn when drawing~$r$ balls uniformly at random without replacement from a set of~$n$ balls, out of which~$k$ are good.
That is, for all $i \in [\max(0, r + k - n) .. \min(k, r)]$ holds $\Pr[X = i] = \binom{k}{i} \binom{n - k}{r - i} / \binom{n}{r}$.
Moreover, $\E[X] = r \frac{k}{n}$ as well as $\Var[X] = r \frac{k}{n} (1 - \frac{k}{n}) \frac{n - r}{n - 1}$.
% Hypergeometric distributions are central to the mutation algorithm. Indeed, if you mutate an individual with $u$ 0-bits, then if you fix $r$ the number of bits you flip, the number of $0$s you flip follows a hypergeometric distribution of parameters $n, u, r$ defined by
%
% \begin{definition}
%     A discrete random variable $X$ is said to follow a hypergeometric distribution with paremeters $n, k, r$ (also noted $X\sim H(n, k, r)$) if $X$ describes the number of good balls drawn when drawing uniformly at random $r$ balls without replacement in a set of $n$ balls, $k$ of which are good.
% \end{definition}
In the context of power-law mutation, $n$ represents the number of bits, $k$ the number of specific bits to flip (for example, $0$-bits), and~$r$ represents the mutation strength.

The following lemma shows that the hypergeometric distribution has a reasonable probability of sampling values that deviate only by the standard deviation from its expected value.
% We will need an interesting lemma: it's well known that a random distribution $X$ is concentrated around its mean, which means the probability mass is concentrated in segments of size $\sqrt{\Var[X]}$ around $\E[X]$. But we need something more precise. More specifically, that the probability of $X$ being equal to a point around the mean is $\Omega(1 / \sqrt{\Var[X]})$. Formally, this is what we want:
\begin{lemma}\label{lem:hypergeo}
    Let $n \in \N$, $k \in [0 .. \frac{n}{2}]$, and $r \in [0 .. \frac{3}{4}n]$, and let $\gamma \in  \R_{> 0}$ be a constant.
    Moreover, let $H \sim \mathrm{Hyp}(n, k, r)$. %with $k< n / 2$ and $r \le 3n/4$,
    Then, for all $x \in [\E[H] - 2 \sqrt{\Var[H]}, \E[H] + 2 \sqrt{\Var[H]}]$ holds that $\Pr[H = x] \ge \gamma / \sqrt{\Var[H]}$.
\end{lemma}

\section{Runtime Analysis}
\label{sec:theory}

We analyze the \moead (\Cref{alg:moead}) on the \oneminmax ($\omm$) function (of dimension $n \in \N_{\geq 1}$) with subproblems spread uniformly across the Pareto front of~$\omm$.
To this end, we make the following assumptions about the input of the algorithm, which we refer to as the \emph{parametrization}:
% We consider power-law mutation with a power-law exponent of $\beta \in (1, 2)$,
We consider decomposition numbers $N \in [n]$, we define the weight vector as $w = (i \frac{n}{N})_{i \in [0 .. N]}$, and we consider the subproblems as defined in \cref{eq:moead-subproblem}. %, and we consider the single neighborhood\todo{I guess that any neighborhood that is at least single is fine?}.
In our calculations, we assume that~$N$ divides~$n$, as this avoids rounding, but we note that all results are equally valid if~$N$ does not divide~$n$, although the computations become more verbose, adding little merit.
We note that, due to space restrictions, some of our proofs are in the appendix.

Our main results are the following, bounding the expected runtime for both standard bit mutation and power-law mutation.

\begin{corollary}
    \label{thm:overall-run-time-standard-bit}
    Consider the \moead maximizing~$\omm$ with standard bit mutation and with the parametrization specified at the beginning of \Cref{sec:theory}.
    Then its expected runtime is $O(n N \log n + n^{n/(2N)} N \log n)$ function evaluations.
\end{corollary}

\begin{corollary}
    \label{thm:overall-run-time-power-law}
    Consider the \moead maximizing~$\omm$ with power-law exponent $\beta \in (1, 2)$ and with the parametrization specified at the beginning of \Cref{sec:theory}.
    Then its expected runtime is $O\left(n N \log n + n^{\beta} \log n\right)$ function evaluations.
    % MOEA-D using a heavy-tailed distribution of parameter $\beta$ with a uniformly spaced population size of $N+1$ solves \cocz in time
    % $$O\left(n N \log n + n^{\beta} \log n\right)$$
\end{corollary}

Both runtime results present two terms, which stem from the two \emph{phases} into which we separate our analysis.
The first term in the results is an upper bound of the first phase, which is the number of function evaluations it takes the \moead to optimize all subproblems.
We call the solutions to these subproblems \emph{$g$-optima}.
Our bounds for either mutation operator for the first phase are the same (\Cref{thm:first-phase}).
The optimization mainly resembles performing $N + 1$ times an optimization similar to that of the well-known $(1 + 1)$ evolutionary algorithm on the \textsc{OneMax} benchmark function.
For $N = n$, our result for standard bit mutation recovers the result by Li et~al.~\cite[Proposition~$4$]{LiZZZ16}.
We go into detail about the analysis in \Cref{sec:theory:first-phase}.

The second phase starts immediately after the first phase and stops once the Pareto front of the \moead covers the Pareto front of~$\omm$.
During this analysis, we only assume that the \moead found the $g$-optima so far.
Thus, in the worst case, it still needs to find all other Pareto-optima of~$\omm$.
To this end, each such optimum needs to be created via mutation \emph{directly} from one of the $g$-optima, as the latter are not being replaced, due to them being optimal.
Depending on the \emph{gap size}, that is, the number of Pareto optima between two $g$-optima of neighboring subproblems, the mutation operator makes a big difference on the expected runtime bound.
We analyze both mutation operators separately in \Cref{sec:theory:second-phase}.

Regarding \Cref{thm:overall-run-time-standard-bit}, we see that the upper bound for standard bit mutation only handles values for $N \in [\frac{n}{2} .. n]$ without any slowdown in comparison to the first phase.
For smaller values, the upper bound is dominated by the second term and becomes super-polynomial once $N = o(n)$.
In \Cref{thm:second-phase-standard-bit}, we prove a lower bound for the second phase that shows that the expected runtime is also at least super-polynomial once $N = o(n)$.
However, as this lower bound only applies to the second phase, it may be possible during the whole optimization of~$\omm$ that the algorithm finds most of the $\omm$-Pareto optima already during the first phase, although we conjecture this not to happen for sufficiently small~$N$.

For power-law mutation (\Cref{thm:overall-run-time-power-law}), the bound for the second phase is independent of~$N$ (\Cref{thm:phase-two}).
This shows that the power-law mutation picks up missing Pareto optima fairly quickly.
In fact, even for an optimal value of $N = n$ for standard bit mutation, which results in a bound of $O(n^2 \log n)$ for the second phase, the bound for power-law mutation is still smaller by a factor of $n^{\beta - 2}$, which is less than~$1$, as we assume that $\beta \in (1, 2)$.
Moreover, for the whole optimization, the upper bound for power-law mutation is best once $N = O(n^{\beta - 1})$, that is, for smaller values of~$N$.
This is in contrast to the upper bound for standard bit mutation, which gets better for larger values of~$N$.

Overall, our results suggest that standard bit mutation profits from having many subproblems, as creating initially skipped solutions may be hard to create once all subproblems are optimized.
In contrast, power-law mutation is slowed down by optimizing many subproblems in parallel.
Instead, it profits from optimizing fewer such subproblems and then creating initially skipped solutions.

In the following, we first analyze the first phase (\Cref{sec:theory:first-phase}) and then the second phase (\Cref{sec:theory:second-phase}).

\subsection{First Phase}
\label{sec:theory:first-phase}
Recall that the first phase considers optimization of~$\omm$ only until all subproblems are optimized, that is, until all $g$-optima are found.
Our main result is the following and shows that finding the $g$-optima is not challenging for the \moead, regardless of the mutation operator.

\begin{corollary}
    \label{thm:first-phase}
    Consider the \moead maximizing~$\omm$ with the parametrization specified at the beginning of \Cref{sec:theory}.
    Then the expected time until~$P$ contains all $g$-optima of~$\omm$ is $O(n N \log n)$ function evaluations for both standard bit mutation and power-law mutation with power-law exponent $\beta \in (1, 2)$.
\end{corollary}

For our proof of \Cref{thm:first-phase}, we split the first phase into two parts.
The first part considers the time until the reference point~$z^*$ is optimal, that is, until $z^* = (n, n)$. %\footnote{We note that, to this end, we implicitly assume that~$z^*$ is initialized with a value in~$\R_{\leq n}^2$. If a component is already larger than~$n$, it is not adjusted anymore.}
For this part, only the optimization of~$g_0$ and~$g_N$ is relevant.

The second part starts with an optimal reference point and considers the time until all $g$-optima are found.
For this part, we consider the optimization of an arbitrary subproblem and multiply the resulting time by roughly $N \log N$, as we consider $N + 1$ subproblems and we wait until all of them are optimized.

In order to prove results that hold for the \moead with standard bit mutation as well as with power-law mutation, we consider a general mutation operator, which we call \emph{general mutation}.
It takes one parameter $p_1 \in (0, 1]$ and works as follows:
It chooses to flip exactly one bit in the parent with probability~$p_1$ (and it flips any other number with probability $1 - p_1$).
Conditional on flipping exactly one bit, it flips one of the~$n$ bits of the parent uniformly at random and returns the result.
Note that standard bit mutation is general mutation with $p_1 = (1 - \frac{1}{n})^{n - 1}$ and that power-law mutation with power-law exponent~$\beta$ is general mutation with $p_1 = 1/C_\beta$.
Both of these values are constants.

For the first part, we get the following result.

\begin{lemma}
    \label{lem:reference-point-optimization}
    Consider the \moead maximizing~$\omm$ with the parametrization specified at the beginning of \Cref{sec:theory} and with general mutation with parameter $p_1 \in (0, 1]$.
    Then the expected time until $z^* = (n, n)$ holds is $O(\frac{n}{p_1} N \log n)$ function evaluations.
\end{lemma}

\begin{proof}
    Let~$T$ be the first iteration such that $z^*_1 = n$.
    Without loss of generality, we only analyze $\E[T]$, as the analysis for $z^*_2 = n$ is identical when changing all~$1$s into~$0$s and vice versa in the individuals in all arguments that follow.
    Thus, the expected runtime that we aim to prove is at most $2\E[T]$ by linearity of expectation.
    Hence, we are left to show that $\E[T] = O(\frac{n}{p_1} \log n)$, where the factor of~$N$ in the bound of \Cref{lem:reference-point-optimization} stems from the \moead making $N + 1$ function evaluations per iteration.
    To this end, we only consider the optimization of~$g_N$.

    By \cref{eq:moead-subproblem}, the choice of the weight vector~$w$, and the definition of~$\omm$, it follows for all $x \in \{0, 1\}^n$ and $z \in \R^2$ that $g_N(x, z) = \max(|z_1 - |x|_1|, 0) = |z_1 - |x|_1|$.
    In each iteration, let~$x_N$ denote the best-so-far solution for~$g_N$ at the beginning of the iteration, and let~$y$ denote its offspring generated via mutation.
    Note that, due to how~$z^*$ is initialized and updated, in each iteration, it holds that $z^*_1 \geq \max(f_1(x_N), f_1(y)) = \max(|x_N|_1, |y|_1)$.
    Thus, if $|y|_1 > |x_N|_1$, then $g(y, z^*) < g(x_N, z)$, and thus~$x_N$ is updated to~$y$ at the end of the iteration.
    Hence, the optimization of~$g_N$ proceeds like a $(1 + 1)$-EA-variant with general mutation optimizing \textsc{OneMax}.

    More formally, let $(X_t)_{t \in \N}$ such that for each $t \in \N$, the value~$X_t$ denotes the number of~$0$s in~$x_N$ at the end of iteration~$t$.
    Note that $X_T = 0$.
    We aim to apply the multiplicative drift theorem (\Cref{thm:multiplicative-drift}) to~$X$ with~$T$.
    By the definition of the mutation operator, it follows for all $t < T$ that $\E[X_t - X_{t + 1} \mid X_t] \geq X_t \frac{p_1}{n}$, since it is sufficient to choose to flip one bit (with probability~$p_1$) and then to flip one of the~$X_t$ $0$s of~$x_N$ (at the beginning of the iteration), which are chosen uniformly at random.
    Thus, by \Cref{thm:multiplicative-drift}, it follows that $\E[T] \leq \frac{n}{p_1} \log n$, concluding the proof.
    \qed
\end{proof}

For the second part, we get the following result.

\begin{lemma}
    \label{lem:finding-all-g-optima}
    Consider the \moead maximizing~$\omm$ with the parametrization specified at the beginning of \Cref{sec:theory} and with $z^* = (n, n)$ and with general mutation with parameter $p_1 \in (0, 1]$.
    Then the expected time until~$P$ contains all $g$-optima of~$\omm$ is $O(\frac{n}{p_1} N \log n)$ function evaluations.
\end{lemma}

\begin{proof}
    Let $i \in [0, N]$.
    We bound with high probability the time~$T$ until~$g_i$ is optimized, only counting the function evaluations for subproblem~$i$.
    The result of \Cref{lem:finding-all-g-optima} follows then by considering the maximum runtime among all values of~$i$ and multiplying it by $N + 1$, as we perform $N + 1$ function evaluations per iteration and optimize all subproblems in parallel.
    We bound the maximum with high probability by taking a union bound over all $N + 1$ different values for~$i$.
    If the maximum of~$T$ over all~$i$ is at least $B \in \R_{\geq 0}$ with probability at most $q \in [0, 1)$, then we get the overall expected runtime by repeating our analysis~$\frac{1}{1 - q}$ times in expectation, as the actual runtime is dominated by a geometric random variable with success probability $1 - q$.
    The overall expected runtime is then $O(B N \frac{1}{1 - q})$
    Thus, it remains to show that $\Pr[T > \frac{n}{p_1} \bigl(\ln(n) + 2\ln(N + 1)\bigr)] \leq (N + 1)^{-2}$, as it then follows that $q \leq (N + 1)^{-1}$ and thus $\frac{1}{1 - q} \leq 2$.
    We aim to prove this probability bound with the multiplicative drift theorem (\Cref{thm:multiplicative-drift}).

    Let $(X_t)_{t \in \N}$ be such that for all $t \in \N$, value~$X_t$ denotes the value of~$g_i(x_i, z^*)$ at the beginning of the iteration.
    Note that $X_T = 0$ and that for all $t < T$ holds that $X_t \in [0 .. n]$ and that for~$X_t$ to reduce (if $X_t > 0$), it is sufficient to flip one of the~$X_t$ bits that reduce the distance.
    Thus, by the definition of the mutation operator, it follows that $\E[X_t - X_{t + 1} \mid X_t] \geq X_t \frac{p_1}{n}$.
    Overall, by (\Cref{thm:multiplicative-drift}), it follows that $\Pr[T > \frac{n}{p_1} \bigl(\ln(n) + 2\ln(N + 1)\bigr)] \leq (N + 1)^{-2}$.
    The proof is concluded by noting that $N \leq n$ and thus $\ln(n) + 2\ln(N + 1) = O(\log n)$.
    \qed
\end{proof}

By the linearity of expectation, the expected time of the first phase is the sum of the expected runtimes of both parts.
Moreover, since standard bit mutation and power-law mutation are both a general mutation with $p_1 = \Theta(1)$, we can omit~$p_1$ in the asymptotic notation.
Overall, we obtain \Cref{thm:first-phase}.

\subsection{Second Phase}
\label{sec:theory:second-phase}
% The second phase can be complicated for the \moead.
% For example, if it is missing the objective value $(\frac{1}{2}\frac{n}{N}, n - \frac{1}{2}\frac{n}{N})$, then the mutation needs to flip at least $\frac{1}{2}\frac{n}{N}$ bits, as the closest $g$-optima are $(0, n)$ and $(\frac{n}{N}, n - \frac{n}{N})$.
% If the \moead uses standard bit mutation, this results in an expected time of $n^{\Omega(n/N)}$ just for finding this single solution.
% In \Cref{sec:experiments}, we see in experiments that, for sufficiently large instances, it becomes very likely that the \moead with standard bit mutation misses optima during the first phase that are hard to generate in the second phase.

Recall that the second phase assumes that the \moead starts with the $g$-optima as its solutions to the subproblems, and it lasts until all $\omm$-Pareto optima are found.

For this phase, the actual objective function is not very important.
All that matters is that if the solutions $(x_i)_{i \in [0 .. N]}$ of the \moead are such that for all $i \in [0 .. N]$ holds that $|x_i|_1 = i \frac{n}{N}$, then~$x_i$ is optimal for~$g_i$.
We refer to such a situation as \emph{evenly spread $g$-optima}.

Since there is a drastic difference between the runtimes of standard bit mutation and power-law mutation, we analyze these two operators separately.

\subsubsection*{Standard Bit Mutation}

Since the standard bit mutation is highly concentrated around flipping only a constant number of bits, it does not perform well when it needs to fill in larger gaps.
The following theorem is our main result, and it proves an upper and a lower bound for the expected runtime of the second phase.
These bounds are not tight, but they show that the runtime is already super-polynomial once $N = o(n)$.

\begin{theorem}
    \label{thm:second-phase-standard-bit}
    Consider the \moead maximizing a bi-objective function with evenly spread $g$-optima and with standard bit mutation, using the parametrization specified at the beginning of \Cref{sec:theory}.
    Moreover, assume that $\frac{n}{2N}$ is integer and that the algorithm is initialized with $(x_i)_{i \in [0 .. N]}$ such that for all $i \in [0 .. N]$ holds that $|x_i|_1 = i \cdot \frac{n}{N}$.
    Then the expected runtime until for each $j \in [0 .. n]$ at least one individual with~$j$~$0$s is created via mutation is $O(n^{n/(2N)} N \log n) \cap \Omega(n^{(1/2)(n/N + 1)} \sqrt{N} 2^{-n/N})$ function evaluations.
\end{theorem}

\subsubsection*{Power-Law Mutation}

The power-law mutation allows to create individuals at larger distance from its parent with far higher probability than standard bit mutation.
Our main result is the following theorem, which shows that the \moead with power-law mutation optimizes the second phase of~$\omm$ efficiently.
As before, we state this theorem in a more general fashion.

\begin{theorem}
    \label{thm:phase-two}
    Consider the \moead optimizing a bi-objective function with evenly spread $g$-optima, using the parametrization specified at the beginning of \Cref{sec:theory}.
    Moreover, assume that the algorithm is initialized with $\{i \cdot \frac{n}{N}\}_{i \in [0 .. N]}$.
    Then the expected runtime until for each $j \in [0 .. n]$ at least one individual with~$j$~$0$s is created via mutation is $O(n^\beta\log n)$.
    % If the population consists of individuals $x_i$ with $||x_i||_0 = i/N$ for $0\le i \le N$, then one finds the whole Pareto Front in $O(n^\beta\log n)$ function evaluations.
\end{theorem}

% The goal of this section is to prove the following theorem

% The two terms in the runtime bound of \Cref{thm:overall-run-time-power-law} resemble the expected runtime of the two phases explained above.
% As the result for the first phase follows from Li et~al.~\cite[Proposition~$4$]{LiZZZ16}, we analyze the second phase in the following.
% We note that our analyses are not specific to~$\omm$ but to \emph{any} bi-objective function whose $g$-optima have the values $\{(n - i \frac{n}{N}, i \frac{n}{N})\}_{i \in [0 .. N]}$.
% We say that such a function has \emph{evenly spread $g$-optima}.
%
% The formula comes from the fact that in a sense, our analysis of \moead is separated in two phases: the first one is about optimizing each $g_i$ in parallel, and this takes $n\log n$ iterations (so $n N \log n$ function evaluations): this was shown in previous papers (TOOD: GIVE REF), then in the second phase, \moead fills in the gaps in the Pareto front that are not covered by the $x_i$s. This takes $n^{\beta} \log n / N$ iterations so $n^\beta \log n$ function evaluations.

The bound of \Cref{thm:phase-two} does not depend on~$N$, in contrast to the bound on the first phase (\Cref{thm:first-phase}).
The reason for our bound not depending on~$N$ is roughly that the effort to fill in a gap between to $g$-optima is inversely proportional to the cost of an iteration, namely, $N + 1$.
Thus, a smaller value of~$N$ leads to faster iterations but more iterations spend to fill the gaps, and vice versa.
% This suggests (keeping in mind though that the bounds do not need to be tight) that smaller values of~$N$ are more beneficial, as the power-law mutation is able to quickly find all solutions that are not find within the first phase.
% This is in strong contrast to using standard bit mutation, where a small value of~$N$ increases the number of optima that could be missed within the first phase, leading thereafter to an expected runtime at least exponential in $n/N$.
% This shows a clear advantage of the power-law mutation, which we analyze empirically in more detail in \Cref{sec:experiments}.

Our (omitted) proof of \Cref{thm:phase-two} makes use of the following lemma, which bounds the probability to create a specific individual from any of the $g$-optima in a single iteration of the \moead.
% Using that lemma, we can prove that one iteration of the algorithm (which amounts to $N+1$ function evaluations) can generate a target individual with reasonable probability compared to the gap size between two individuals of the population.

\begin{lemma}\label{lem:gen_iter}
    Consider a specific iteration during the optimization of the \moead of a bi-objective function with evenly spread $g$-optima, using the parame\-trization specified at the beginning of \Cref{sec:theory}.
    Moreover, assume that the algorithm is initialized with $\{i \cdot \frac{n}{N}\}_{i \in [0 .. N]}$.
    Last, let $y \in \{0, 1\}^n$ be such that $u \coloneqq |y|_0 \in [0 .. \frac{n}{2}]$.
    Then the probability that~$y$ is produced during mutation in this iteration is $\Omega\left(N(n^{-\beta})\right)$.
    % If the population consists of indivuals $x_i$ with $||x_i||_0 = i\cdot n /N$ for $0\le i \le N$, then if $u \le n/2$, during one iteration of the algorithm, one of the $x_i$'s will produce a bitstring with $u$ 0s with probability
    % $$\Omega\left(N(n^{-\beta})\right)$$
\end{lemma}

\begin{proof}
    Let $i$ be such that $ni/N < u\le n(i+1)/N$. Clearly, $i\le N/2$, as $u\le n/2$. Note that are at least $i/4$ values of $j\in [0..N/4]$ such that $jn/N\le n/4$. By \Cref{lem:trans_prob}, each individual $x_j$ mutates into an individual with $u$ 0-bits with probability $\Omega\bigl(\bigl((i + 1)n/N\bigr)^{-\beta}\bigr)$.
    Because there are $\Omega(i)$ such possible mutations during an iteration, the probability of generating at least one during an iteration is $\Omega\left( i^{1 - \beta} n^{-\beta} \cdot N^{\beta} \right) = \Omega(N \bigl(\frac{N}{i}\bigr)^{\beta - 1}n^{-\beta}) = \Omega\left(N n^{-\beta}\right)$, concluding the proof.
    \qed
\end{proof}

\section{Conclusion}
\label{sec:conclusion}

We studied the impact of the decomposition number~$N$ of the \moead~\cite{ZhangL07} on the classic multi-objective benchmark \oneminmax ($\omm$)~\cite{GielL10} theoretically.
Our analyses considered subproblems that are evenly spread out on the Pareto front of~$\omm$.
Especially, we studied the expected runtime when starting with all optima of the subproblems (the \emph{$g$-optima}) and requiring to find the remaining Pareto optima of~$\omm$.

One of our theoretical results (\Cref{thm:phase-two}) shows that using power-law mutation allows the \moead to efficiently find the entire Pareto front of~$\omm$ even if it is initialized only with the $g$-optima.
Interestingly, this bound is independent of the number of problems~$N$ and thus the number of initially missing Pareto optima between two neighboring $g$-optima.
Together with our general bound for finding all $g$-optima (\Cref{thm:first-phase}), this shows that the \moead with power-law mutation always optimizes~$\omm$ efficiently (\Cref{thm:overall-run-time-power-law}).
Depending on~$N$, our total-runtime bound ranges from $O(n^2 \log n)$ in the worst case to $O(n^\beta \log n)$ in the best case of $N = O(n^{\beta - 1})$.
% Parts of this result even generalize to other multi-objective functions where the optima of the subproblems coincide with the ones of~$\omm$ (\Cref{thm:phase-two}).
This suggests that the \moead is, in fact, slowed down when using many subproblems, despite large values of~$N$ implying that the subproblems cover the Pareto front of~$\omm$ better than smaller values.
The reason is that a large value of~$N$ roughly translates to optimizing the same problem~$N$ times.
With the power-law mutation, it is better to optimize fewer and therefore more diverse subproblems and to then find the remaining Pareto optima efficiently via the power-law mutation.

For standard bit mutation, when starting with all $g$-optima, we show (\Cref{thm:second-phase-standard-bit}) that the \moead is not capable of efficiently finding all Pareto optima if~$N$ is sufficiently small, as our lower bound is super-polynomial for $N = o(n)$.
% For a part of the optimization process, we prove a lower bound (\Cref{thm:second-phase-standard-bit}) that shows when and how quickly the runtime deteriorates with~$N$.
Nonetheless, for $N = \Theta(n)$, the expected runtime of the \moead with standard bit mutation is polynomial in this part.
This translates to an overall polynomial expected runtime (\Cref{thm:overall-run-time-standard-bit}) for $N = \Theta(n)$ that it is even $O(n^2 \log n)$ for $N \in [\frac{n}{2} .. n]$, matching our worst-case bound with power-law.
% In this case, the optimization hinges on the fact that the \moead finds several Pareto-optima already while optimizing the subproblems.
% However, our empirical results show\todo{Provide a reference.} that the algorithm, indeed, misses some solutions for larger problem sizes, which it cannot find afterward within reasonable time.
% In contrast, and in line with our theoretical results, our empirical results for power-law mutation show that a small number~$N$ of subproblems results indeed in smaller runtimes.
% In fact, the empirical runtimes for finding remaining solutions is even smaller than our theoretical upper bound (\Cref{thm:phase-two}), suggesting an even stronger, more beneficial impact of the power-law mutation.

Overall, our results suggest a clear benefit of power-law mutation over standard bit mutation of the \moead in the setting we considered.
Not only is the power-law variant faster, it is also far more robust to the choice of~$N$ and thus to how the problem is decomposed.

For future work, it would be interesting to improve the upper bounds or prove matching lower bounds.
A similar direction is to consider an exchange of best-so-far solutions among the subproblems.
The classic \moead supports such an exchange, which could potentially lead to finding the $g$-optima more quickly.
Another promising direction is the study of different problem decompositions, for example, not-evenly spread subproblems or subproblem definitions different from \cref{eq:moead-subproblem}.
Last, we considered the~$\omm$ setting, with a stronger generalization some of our results (\Cref{thm:second-phase-standard-bit,thm:phase-two}).
However, it is not clear to what extent the benefit of power-law mutation carries over to problems with an entirely different structure, such as \textsc{LeadingOnesTrailingZeros}.

\bibliographystyle{splncs04}
\bibliography{ich_master,alles_ea_master,rest}

% \bibitem[1]{ZZZ}
% Y. -L. Li, Y. -R. Zhou, Z. -H. Zhan and J. Zhang, "A Primary Theoretical Study on Decomposition-Based Multiobjective Evolutionary Algorithms," in IEEE Transactions on Evolutionary Computation, vol. 20, no. 4, pp. 563-576, Aug. 2016, doi: 10.1109/TEVC.2015.2501315. keywords: {Runtime;Optimization;Evolutionary computation;Algorithm design and analysis;Complexity theory;Sociology;Statistics;Decomposition-based multiobjective evolutionary algorithms (MOEAs);runtime analysis;theoretical study},

\newpage
\appendix
\section{Appendix of Paper 252---Proven Runtime Guarantees for How the \moead Computes the Pareto Front From the Subproblem Solutions}
This appendix contains all of the proofs that are omitted in the main paper, due to space restrictions.
It is meant to be read at the reviewer's discretion only.

\subsection*{Proofs of \Cref{sec:moead}}

The proof of \Cref{lem:trans_prob} makes use of the anti-concentration of the hypergeometric distribution (\Cref{lem:hypergeo}), which is why we defer the proof of \Cref{lem:trans_prob} to the end of this section.

In order to prove \Cref{lem:hypergeo}, we show that the ratio of neighboring probabilities of a hypergeometrically distributed random variable is close to~$1$.
% To prove this, we will first prove that the probability distribution doesn't fluctuate too much around the mean, which is shown in the following lemma
\begin{lemma}
\label{lem:ratio_bound}
    Let $n \in \N$, $k \in [0 .. \frac{n}{2}]$, and $r \in [0 .. \min(k, \frac{3}{4}n)]$, and let $C \in  \R_{> 1}$ be a constant.
    % There exists a universal constant $C > 1$ and a universal integer $p$ such that for all $H \sim H(n, k, r)$ with $k < n/2$ and $r \le 3n/4$, for all $i\in (-2\sqrt{\Var[H]}..2\sqrt{\Var[H]})$, one has
    Moreover, let $H \sim \mathrm{Hyp}(n, k, r)$, $i \in [-2\sqrt{\Var[H]} - 1 .. 2\sqrt{\Var[H]} + 1]$, and let $\delta_i = \Pr[H = \E[H] + i + 1] / \Pr[H = \E[H] + i]$.
    Then
    \begin{align*}
        (1 -7/\sqrt{\Var[H]}) \le \delta_i \le (1 + 7/\sqrt{\Var[H]}) .
    \end{align*}
    % Where
    % $$\delta_i := \frac{\Pr[H = \frac{rk}{n} + i+1]}{\Pr[H = \frac{rk}{n} + i]}$$
\end{lemma}

\begin{proof}[of \Cref{lem:ratio_bound}]
Take $H$ and $i$ as in the statement of the lemma, and let $\sigma \sqrt{\Var[H]}$. Notice that $\sigma^2 = r \frac{k}{n} \left(1 - \frac{k}{n}\right) \frac{n - r}{n-1}$ so that $\sigma^2 \le \frac{rk}{n}$.

Suppose that $i\ge 0$. We lower-bound $\delta_i$ by an expression of the desired form

\begin{align*}
    \delta_i &=  \frac{\left(k - r\frac{k}{n} + i\right)\left(r - \frac{rk}{n} - i\right)}{\left(\frac{rk}{n} + i + 1\right)\left(n - k - r + \frac{rk}{n} + i + 1\right)}.\\
    \intertext{In this expression, factorize by $k$ in the top left expression, by $r$ in the top right expression, by $\frac{rk}{n}$ in the bottom left expression and for the bottom right one, notice that $n - k - r + \frac{rk}{n} = \frac{1}{n}(n^2 -nk - nr + rk) = \frac{1}{n}(n - k)(n - r)$ to obtain}
    \delta_i&= \frac{k\left(1 - \frac{r}{n} + \frac{i}{k}\right)r\left(1 - \frac{k}{n} - \frac{i}{r}\right)}{\frac{rk}{n}\left( 1 + \frac{i+1}{\frac{rk}{n}} \right)\left(\frac{1}{n}(n - k)(n - r) + i + 1  \right)}.\\
    \intertext{Here, multiply the top two expressions by $n$ using the $\frac{1}{n}$ terms in the denominator and obtain}
    \delta_i &= \frac{\left(n - r + \frac{ni}{k}\right)\left(n - k - \frac{ni}{r}\right)}{\left(1 + \frac{i+1}{\frac{rk}{n}}\right)\big( (n - k)(n - r) + n(i + 1)  \big)}.\\
    \intertext{Dividing the numerator and denominator by $(n-k)(n-r)$ and distributing the factors properly yields}
    \delta_i &= \frac{\left(1 + \frac{ni}{k(n-r)}\right)\left(1 - \frac{ni}{r(n-k)}\right)}{\left(1 + \frac{i+1}{\frac{rk}{n}}\right)\left( 1 + n\frac{i+1}{(n-k)(n-r)} \right)}.\\
    \intertext{Now remember that $k\le n/2$ which implies $n - k \ge k$. Use that in the top right expression to get }
    \delta_i &\ge \frac{\left(1 + \frac{ni}{k(n-r)}\right)\left(1 - \frac{ni}{rk}\right)}{\left(1 + \frac{i+1}{\frac{rk}{n}}\right)\left( 1 + n\frac{i+1}{k(n-r)} \right)}.
    \intertext{The top right term is at least $1 - 1/\sigma$ from $i\le \sigma$ and $rk/n\ge \sigma^2$ so that}
    \delta_i &\ge \frac{\left(1 + \frac{ni}{k(n-r)}\right)\left(1 - \frac{1}{\sigma}\right)}{\left(1 + \frac{i+1}{\frac{rk}{n}}\right)\left( 1 + n\frac{i+1}{k(n-r)} \right)}.\\
    \intertext{Similarly, the bottom left term is at most $(1 + 1/\sigma)$ so that}
    \delta_i &\ge \frac{\left(1 + \frac{ni}{k(n-r)}\right)\left(1 - \frac{1}{\sigma}\right)}{\left(1 + 1/\sigma\right)\left( 1 + n\frac{i+1}{k(n-r)} \right)}.\\
    \intertext{Now use the fact that $\frac{1 - 1/\sigma}{1 + 1/\sigma} = \frac{\sigma-1}{\sigma+1} \ge 1 - 3 / \sigma$ to get}
    \delta_i &\ge (1 - 3/\sigma)\frac{\left(1 + \frac{ni}{k(n-r)}\right)}{\left( 1 + n\frac{i+1}{k(n-r)} \right)}\\
    &= \left(1 - \frac{3}{\sigma}\right) \frac{1 + \frac{ni}{k(n-r)}}{1 + \frac{ni}{k(n-r)} +\frac{n}{k(n-r)}}.
    \intertext{Divide the denominator by the numerator to obtain}
    \delta_i &\ge \left(1 - \frac{3}{\sigma}\right) \frac{1}{1 + \frac{n}{k(n-r)} \frac{1}{1 + \frac{ni}{k(n-r)}}}\\
    &=  \left(1 - \frac{3}{\sigma}\right) \frac{1}{1 + \frac{n}{k(n-r) + ni}}.
    \intertext{Because $r\le 3n/4$, it is true that $n-r\ge r/3$ and because $i\ge 0$, we obtain}
    \delta_i &\ge \left(1 - \frac{3}{\sigma}\right) \frac{1}{1 + \frac{3n}{kr}}\\
    &\ge \left(1 - \frac{3}{\sigma}\right) \frac{1}{1 + 3/\sigma}\\
    &\ge \left(1 - \frac{3}{\sigma}\right)^2\\
    &\ge \left(1 - \frac{6}{\sigma}\right),
\end{align*}
which is what we wanted.

For the upper bound, we directly upper-bound $\delta_i$ in the following way, starting like we did for the lower-bound, that is,
\begin{align*}
    \delta_i &=  \frac{\left(k - r\frac{k}{n} + i\right)\left(r - \frac{rk}{n} - i\right)}{\left(\frac{rk}{n} + i + 1\right)\left(n - k - r + \frac{rk}{n} + i + 1\right)}\\
     &= \frac{\left(1 + \frac{ni}{k(n-r)}\right)\left(1 - \frac{ni}{r(n-k)}\right)}{\left(1 + \frac{i+1}{\frac{rk}{n}}\right)\left( 1 + n\frac{i+1}{(n-k)(n-r)} \right)}.
     \intertext{Now upper-bound the top right term by 1 and lower-bound the terms in the denominator by 1 to obtain}
     \delta_i &\le \left(1 + \frac{ni}{k(n-r)}\right).
     \intertext{Because $r\le 3/4n$, $n-r \ge r/3$ so that}
     \delta_i &\le \left(1 + 3\frac{ni}{kr}\right)\\
              &\le \left(1 + \frac{6}{\sigma}\right).
\end{align*}
The same calculations work for negative $i$ as well, which concludes the proof.
\qed
\end{proof}

\subsubsection*{Proof of \Cref{lem:hypergeo}}
% Now back to the proof of ~\cref{lem:hypergeo}
\begin{proof}[of \Cref{lem:hypergeo}]
Recall that $\mu \coloneqq \E[H] = r \frac{k}{n}$, and that $\sigma^2 \coloneqq \Var[H] = r\frac{k}{n}\left(1 - \frac{k}{n}\right)\frac{n - r}{n-1}$.

From the expression of $\sigma^2$, it is clear that $\sigma^2\le \frac{rk}{n}$. Because $r\le 3n/4$ and $k\le n/2$, it follows that $\sigma^2 \ge \frac{rk}{8n}$, so that $\sigma^2$ is of the order of $rk/n$.

We aim to show two properties: The first is that $H$ is concentrated around $\mu$, the second is that the probability distribution does not decrease too quickly for values larger than~$\mu$. The first property follows from Chebyshev's inequality, as
$$ \Pr\left[\left|H - \mu\right|\le 2\sigma\right]\ge 3/4.$$
The second property follows from \Cref{lem:ratio_bound}, by which we have that
$$\left(1 - 7 / \mu\right) \le \delta_i \le \left(1 + 7 / \mu\right),$$
where $\delta_i = \frac{\Pr[H = \mu + i+1]}{\Pr[H = \mu + i]}$.

Now notice that if $\mu - 2\sigma \le u < v \le \mu + 2\sigma$, one has
$$\Pr[H = v] = \Pr[H = u] \cdot \prod_{i=u}^{v-1} \delta_i$$
Hence,
$$ (1 - 7/\sigma)^{v - u} \cdot \Pr[H = u] \le \Pr[H = v] \le (1 + 7/\sigma)^{v - u}\Pr[H = u].$$
Because $v - u\le 4\sigma$, we obtain
$$e^{-28}/28\Pr[H = u] \le \Pr[H = v] \le e^{28} \Pr[H = u].$$
And because $\Pr[H \in [\mu - 2\sigma, \mu +  2\sigma]] \ge 3/4$, we have
$$\max_{i \in [-2\sigma.. 2\sigma]} \Pr[H = \mu + i] \ge \frac{3}{16\sigma}.$$
Thus, owing to to the comparison between any pair of probabilities in that interval that we just showed, we have that for all $i\in [-2\sigma..2\sigma]$ that
$$\Pr[H = \mu + i] \ge \frac{\eulerE^{-28}}{28} \frac{3}{16\sigma},$$
which concludes the proof.
\qed
\end{proof}

\subsubsection*{Proof of \Cref{lem:trans_prob}}

\begin{proof}[of \Cref{lem:trans_prob}]
    Let $\Delta$ be the number of zeros that are flipped when going from $u$ to $v$ zeros. Then $v - u + \Delta$ 1s have to be flipped. This happens with probability
    $$ \left(2\Delta + v - u \right)^{-\beta} \frac{\binom{u}{\Delta}\binom{n - u}{v - u + \Delta}}{\binom{n}{2 \cdot \Delta + v - u}} \eqqcolon \left(2\Delta + v - u\right)^{-\beta} p_{\Delta} .$$
    The fraction $p_{\Delta}$ corresponds to the probability of getting $\Delta$ successes in a hypergeometric distribution
    $H_\Delta\sim \mathcal{H}(n, u, 2\Delta + v - u)$.

    Notice that $\E[H_\Delta] = (2\Delta + v - u) \frac{u}{n}$.
    Hence, $\E[H_\Delta] - \Delta = (v-u)\frac{u}{n} - (1 - 2u/n)\Delta$.
    If $\Delta = \frac{1}{1 - 2u/n}\frac{(v-u)u}{n} \eqqcolon \Delta_0$ then $\E[H_\Delta] - \Delta=0$ .

    We now bound $\Delta_0$.
    To this end, let $s = n/2 - u$ and $t = n/2 - v$. Since $s-t\le s$ and $n/2 - s\le n/2$, we have
    $$\Delta_0 = \frac{(s - t)(n/2 - s)}{2s}\le n/4.$$
    Let $\sigma = \sqrt{\Var[H_{\Delta_0}]}$. Notice that
    $$ \sigma^2 = (2\Delta_0 + v-u) \frac{u}{n}\left(1 - \frac{u}{n}\right) \frac{n - 2\Delta_0 - (v-u)}{n - 1}.$$
    Thus, it follows that $\sigma \le \sqrt{2\Delta_0+v-u} \le \sqrt{n}$.
    Hence, we obtain
\begin{align*}
    \left|\E[H_{\Delta_0 + i}] - (\Delta_0 + i)\right| &= \left|\E[H_{\Delta_0}] - \Delta_0 + 2iu/n - i\right|\\
    &= \left|(1 - 2u/n)i\right|\\
    &\le |i|.
\end{align*}
    Consequently, if $0\le i \le \alpha\sigma$, we have $\Delta_0 + i \le (1/4 + \alpha) n$. For the remainder, we suppose that $\alpha \le 1/2$. We get $\Delta_0 + i \le 3/4n$. By \Cref{lem:hypergeo}, with its constant~$\gamma$, the probability of getting from $u$ to $v$ with $\Delta_0 + i$ is at least
    $$ \left(2\Delta_0 + 2i + v - u\right)^{-\beta} \frac{\gamma}{\sqrt{\Var[H_{\Delta_0 + i}]}} .$$

    We now bound the number of bits to flips $2\Delta_0 + 2i + v-u$. First notice that $i\le \sigma \le \sqrt{v - u + \delta} \le v - u + \delta$ so that
\begin{align*}
    2\Delta_0 + 2i + v - u &\le 4\Delta_0 + 3(v - u)\\
    &= \frac{4(v - u) \frac{u}{n} + 3(1 - 2u/n)(v - u)}{1 - 2u/n}\\
    &= 4\frac{v-u}{1 - 2u/n}\left(\frac{u}{n} + 3 - 6\frac{u}{n}\right)\\
    &= 4(v-u) \frac{1}{1 - 2u/n}\left( 3 - 5\frac{u}{n}\right)\\
    &\le 12(v-u) \frac{1}{1 - 2u/n}.\\
    \intertext{Because $u\le n/4$, we have $2u/n \le 1/2$ so that}
    &\le 24(v-u).
\end{align*}
    The expression of $\Var[H_{\Delta_0 + i}]$ is
    $$ (2\Delta_0 + 2i + v-u) \frac{u}{n}\left(1 - \frac{u}{n}\right) \frac{n - 2\Delta_0 - 2i - (v - u)}{n}.$$
    Note that this expression is at most quadratic in~$i$ independently of $u,v,n,N$. This implies that the square root of this expression is at most linear. From this it follows that there exists a constant $\kappa > 0$ such that for $i\le \alpha \sigma$, we have $\sqrt{\Var[H_{\Delta_0 + i}]}\le \kappa \sigma$. It follows that the probability of getting from $u$ to $v$ with $\Delta_0 + i$ is at least
    $$ \frac{(8(v-u))^{-\beta}\gamma}{\kappa \sigma}.$$
    Using that there are at least $\alpha\sigma$ possible values of $i$ to use, we get that the probability to get from $u$ to $v$ is at least $(8(v-u))^{-\beta}\alpha\beta / \kappa$.
    \qed
\end{proof}

\subsection*{Proofs of \Cref{sec:theory}}

\subsubsection*{Proof of \Cref{thm:second-phase-standard-bit}}

\begin{proof}[of \Cref{thm:second-phase-standard-bit}]
    Let $b \coloneqq \frac{n}{2N}$ denote the Hamming distance maximum distance of a non-$g$-optimal individual to a $g$-optimum.
    We first prove the \textbf{upper bound.}
    For all $i \in [0 .. N - 1]$, we first bound with high probability the number of iterations~$T_1$, counting only the function evaluations of subproblem~$g_i$, until for all $j \in [b]$, the mutation with parent~$x_i$ created an individual with $i\frac{n}{N} + j$ $1$s.
    Taking the maximum of~$T_1$ among all values of~$i$ and multiplying the result by $N + 1$ results in an upper bound for covering the left side of all gaps between two $g$-optima.
    Afterward, we do the same for all $i' \in [N]$ for the number of iterations~$T_2$, counting only the function evaluations of subproblem~$g_{i'}$, until for all $j \in [b]$, the mutation with parent~$x_i$ created an individual $i\frac{n}{N} - j$ $1$s.
    Taking again the maximum over all values for~$i'$ and multiplying the result by $N + 1$ results in an upper bound for covering the right side of all gaps between two $g$-optima.

    We estimate the maximum over the respective runtimes in the same way as in the proof of \Cref{lem:finding-all-g-optima}, that is, we find a general upper bound $B \in \R_{\geq 0}$ that holds with probability at least $1 - q \in (0, 1]$ for all~$T_1$ for all choice of~$i$ (and analogous for~$T_2$).
    We do so via the multiplicative drift theorem (\Cref{thm:multiplicative-drift}).
    Afterward, we take a union bound over all~$N$ values of~$i$, showing that the maximum over all~$T_1$ is at most~$B$ with probability at least $1 - Nq$.
    Since the overall runtime for covering the left half of each gap is dominated by a geometric random variable with success probability $1 - Nq$, the expected runtime for this part is at most $\frac{1}{1 - Nq} B N$ function evaluations.
    The overall expected runtime of both parts is by linearity of expectation then the sum of two expected values of each part.

    We start with the \textbf{left half of each gap.}
    Let $i \in [0 .. N - 1]$, and let $j \in [b]$.
    In order to create an individual with $i\frac{n}{N} + j$ $1$s, it is sufficient to flip~$j$ specific~$0$s of~$x_i$ into~$0$s.
    Hence, the probability to create such an individual is at least $n^{-j} (1 - \frac{1}{n})^{n - j} \geq \eulerE^{-1} n^{-j} \geq \eulerE^{-1} n^{-b}$.
    Let $(X_t)_{t \in \N}$ be such that for all $t \in \N$ holds that~$X_t$ desired individuals have not been created yet (that is,~$X_t$ values of~$j$ are still missing).
    Note that $X_{T_1} = 0$ and that for all $t < T_1$ holds that $\E[X_t - X_{t + 1} \mid X_t] \geq \eulerE^{-1} n^{-b}$.
    Thus, by (\Cref{thm:multiplicative-drift}) follows that $\Pr[T_1 > \eulerE n^b \bigl(\ln(b) + 2\ln(N + 1)\bigr)] \leq (N + 1)^{-2}$.
    Hence, via a union bound over all~$N$ choices of~$i$, it follows that the maximum~$T_1$ is at most $\eulerE n^b \bigl(\ln(b) + 2\ln(N)\bigr)$ with probability at least $1 - (N + 1)^{-1} \geq \frac{1}{2}$.
    Hence, the overall expected runtime of this part is $O(n^b N \ln(bN))$.

    For the \textbf{right half of each gap,} the arguments are the same after restricting $i \in [N]$ and writing~$-j$ instead of~$+j$.

    Overall, noting that $bN \leq n$, the expected runtime for filling each half is thus $O(n^b N \log n)$, concluding the proof for the upper bound.

    For the \textbf{lower bound,} we only bound the expected time it takes to create an individual~$y$ with $|y|_1 = b$.
    To this end, we first bound the probability when choosing~$x_1$ as parent, denoted by~$q$.
    Recall that $|x_1|_1 = 2b$.
    In order to create~$y$, the mutation needs to flip at least~$b$ of the~$2b$ $1$s of~$x_1$.
    Hence, $q \leq n^{-b} \binom{2b}{b}$.
    Since it holds that $\binom{2b}{b} = O(2^{2b} / \sqrt{b})$, it follows that $q \leq n^{-b} \cdot 2^{2b} / \sqrt{b} \eqqcolon q'$.
    For the probability to create~$y$ when choosing~$x_0$ as parent, we also use~$q'$ as an upper bound.
    For all other choices of parents, the probability to create~$y$ is at most~$o(q')$, as at least $\frac{3}{2}b$ bits need to be flipped, which scales exponentially to the base~$n$ whereas the binomial coefficient involved only scales exponentially with base~$2$.
    Hence, over all, the probability~$p$ to create~$y$ within~$N$ function evaluations is at most $2q'\bigl(1 + o(1)\bigr)$.

    The event that~$y$ is created follows a geometric distribution with success probability~$p$.
    Due to $p \leq 2q'\bigl(1 + o(1)\bigr)$, it follows that the expected number of iterations to create~$y$ is $\Omega(\frac{1}{q'}) = \Omega(n^b \sqrt{b} 2^{-2b})$.
    Multiplying this expression by $N + 1$ in order to account for the $N + 1$ fitness evaluations each iteration concludes the proof.
    \qed
\end{proof}

\subsubsection*{Proof of \Cref{thm:phase-two}}

For the proof of \Cref{thm:phase-two}, we make use of the following theorem.

\begin{corollary}[{\cite[Corollary~$2.2$]{Janson17GeometricTailBound}}]
    \label{thm:geometric-tail-bound}
    Let $n \in \N_{\geq 1}$, let $(p_i)_{i \in [n]} \in (0, 1]^n$, and let $(X_i)_{i \in [n]}$ be independent and such that for all $i \in [n]$ holds that~$X_i$ follows a geometric distribution with success probability~$p_i$.
    Last, let $Y = \sum_{i \in [n]} X_i$.
    Then it holds for all $r \in \R_{\geq 1}$ that $\Pr[Y \geq r \E[Y]] \leq r \eulerE^{1 - r}$.
\end{corollary}

\begin{proof}[of \Cref{thm:phase-two}]
By symmetry, we just have to prove that we can get individuals with $u \le n/2$ bits, as the runtime to get individuals with $u \ge n/2$ 0-bits is the same when changing~$0$s and~$1$s. Thus, the runtime is at most twice that of the runtime needed to get all individuals with $u\le n/2$ $0$s.

Let $u\le n/2$. By \Cref{lem:gen_iter}, an iteration gives $u$ with probability $\Omega\left(N n^{-\beta}\right)$.
Hence, the expected number of iterations required to get $u$ is $O\left(n^{\beta} / N\right)$.
Multiplying by $N$, we get that the expected number of function evaluations to obtain $u$ is $O(n^{\beta})$.

For $u\in [0..n]$, let $X_u$ be the random variable giving the number of function evaluations before $u$ is obtained after the first phase of the analysis is finished. By \Cref{thm:geometric-tail-bound}, since~$X_u$ is a geometric variable, we see for all $\lambda \in \R_{\geq 1}$ that
$$\Pr[X_u \ge \lambda\log (n) n^{\beta}] = O\left(\frac{
\log n}{n^\lambda}\right).$$

We aim to show that $\E[\max(X_0,\ldots, X_n)] = O(n^{\beta}\log n)$. First notice that
$$\E[\max\nolimits_u(X_u)] \le \sum\nolimits_{u} \E[X_u] = O(n^{\beta + 1}\log n)$$
Taking $\lambda = \beta + 3$ in the previous tail bound and using a union bound on the events $\{X_u \ge \lambda \log( n) n^\beta\}$, we obtain
$$\Pr[\max\nolimits_u(X_u)\ge (\beta + 3)\log(n) n^{\beta}] = O\left(\frac{\log n}{n^{\beta + 3}}\right) = o\left(n^{-\beta - 2}\right).$$
Decomposing the random variable $\max_u(X_u)$ according to $\max_u(X_u)\le (\beta + 3)\log(n) n^{\beta}$ or the complement, we obtain
$$\E[\max\nolimits_u(X_u)] \le (\beta + 3)n^{\beta}\log n + o\left(n^{-\beta-2}\right)\E[\sum\nolimits_{u} X_u] = O(n^{\beta} \log(n)).$$
Hence, it follows that the average number of function evaluations to get the whole Pareto Front is $O\left(n^{\beta} \log n\right)$.
\qed
\end{proof}

\end{document}